\documentclass{article}

\usepackage{microtype}
\usepackage{graphicx}
\usepackage{subfigure}
\usepackage{booktabs} 
\usepackage{amsthm,amsmath,amssymb}
\usepackage{thm-restate}
\usepackage{multirow}
\usepackage{amsfonts}
\newcommand{\Diff}{\textnormal{J}}
\newcommand{\Hess}{\textnormal{H}}
\newcommand{\indep}{\rotatebox[origin=c]{90}{$\models$}}
\newcommand{\Cov}{\mathrm{Cov}}

\usepackage{hyperref}

\newtheorem{lemma}{Lemma}

\newtheorem{theorem}{Theorem}

\usepackage[accepted]{icml2020}

\icmltitlerunning{A Critical View of the Structural Causal Model}

\begin{document}

\twocolumn[
\icmltitle{A Critical View of the Structural Causal Model}

\icmlsetsymbol{equal}{*}

\begin{icmlauthorlist}
\icmlauthor{Tomer Galanti}{tau}
\icmlauthor{Ofir Nabati}{tau}
\icmlauthor{Lior Wolf}{tau,fb}
\end{icmlauthorlist}

\icmlaffiliation{tau}{School of Computer Science, Tel Aviv University, Israel}
\icmlaffiliation{fb}{Facebook AI Research (FAIR)}

\icmlcorrespondingauthor{Tomer Galanti}{tomerga2@post.tau.ac.il}
\icmlcorrespondingauthor{Lior Wolf}{wolf@fb.com}

\vskip 0.3in
]

\printAffiliationsAndNotice{
} 

\begin{abstract}
In the univariate case, we show that by comparing the individual complexities of univariate cause and effect, one can identify the cause and the effect, without considering their interaction at all. In our framework, complexities are captured by the reconstruction error of an autoencoder that operates on the quantiles of the distribution. Comparing the reconstruction errors of the two autoencoders, one for each variable, is shown to perform surprisingly well on the accepted causality directionality benchmarks. Hence, the decision as to which of the two is the cause and which is the effect may not be based on causality but on complexity.

In the multivariate case, where one can ensure that the complexities of the cause and effect are balanced, we propose a new adversarial training method that mimics the disentangled structure of the causal model. We prove that in the multidimensional case, such modeling is  likely to fit the data only in the direction of causality. Furthermore, a uniqueness result shows that the learned model is able to identify the underlying causal and residual  (noise) components. Our multidimensional method outperforms the literature methods on both synthetic and real world datasets.
\end{abstract}

\section{Introduction}

A long standing debate in the causality literature, is whether causality can be inferred without intervention~\citep{pearl,Spirtes2000}. The Structural Causal Model (SCM)~\citep{Spirtes2000} is a simple causative model for which many results demonstrate the possibility of such inference~\citep{NIPS2010_4173,pmlr-v84-bloebaum18a,goudet:hal-01649153,ncc,towards}. In this model, the effect ($Y$) is a function of the cause ($X$) and some independent random noise ($E$).

In this work, we take a critical perspective of the univariate SCM. We demonstrate empirically that for the univariate case, which is the dominant case in the existing literature, the SCM leads to an effect that has a lower complexity than the cause. Therefore, one can identify the cause and the effect, by measuring their individual complexities, with no need to make the inference based on both variables simultaneously. Thus, the decision as to which of the two is the cause and which is the effect may not be based on causality but on complexity.

Since we are dealing with unordered univariate random variables, the complexity measure has to be based on the probability distribution function. As we show empirically, comparing the entropies of the distribution of two random variables is ineffective for inferring the causal direction. We, therefore, consider the quantiles, i.e, fixed sized vectors that are obtained as sub-sequences of the sorted sampled values of the variable. 

We consider suitable complexity scores for these vectors. In our analysis, we show that the reconstruction error of an autoencoder of a multivariate random variable is a valid complexity measure. In addition, we link the reconstruction error based complexity, in the case of variational autoencoders, to the differential entropy of the input random variable. Hence, by computing the reconstruction errors of trained autoencoders on these vectors, we estimate the entropies of the quantile vectors of $X$ and $Y$.

The challenges of measuring causality independently of complexity in the 1D case lead us to consider the multidimensional case, where the complexity can be controlled by, for example, manipulating the dimension of the noise signal in the SCM. Note that unlike~\cite{goudet:hal-01649153}, we consider pairs of multivariate vectors and not many univariate variables in a graph structure. We demonstrate that for the multidimensional case, any method that is based on comparing the complexity of the individual random variables $X$ and $Y$ fails to infer causality of random variables. Furthermore, we extend a related univariate result by~\cite{pmlr-v6-zhang10a} to the multidimensional case and prove that an SCM is unlikely to hold in both directions $X \to Y$ and $Y \to X$.

Based on our observations, we propose a new causality inference method for multidimensional cause and effect. The algorithm learns three networks in a way that mimics the parts of the SCM. The noise part is unknown and is replaced by a function that is constrained to be independent of the cause, as captured by an adversarial loss. However, we show empirically that even without the explicit constraint, in several cases, such an independence emerges.

Our empirical results support our analysis and demonstrate that in the univariate case, assigning cause and effect based on complexity is competitive with the state of the art methods. In the multidimensional case, we show that the proposed method outperforms existing multivariate methods, as well as new extensions of univariate literature methods.

\subsection{Problem Setup}\label{sec:setup}
We investigate the problem of causal inference from observational data. A non-linear structural causal model (SCM for short) is a generative process of the following form:
\begin{equation}\label{eq:causal}
\begin{aligned}
X &\sim \mathbb{P}_X \\
E &\sim \mathbb{P}_E \\
Y &\leftarrow g(f(X),E)
\end{aligned}
\end{equation}
The functions $g:\mathbb{R}^{d_f+d_e} \to \mathbb{R}^{d_y}$ and $f:\mathbb{R}^{d_x} \to \mathbb{R}^{d_f}$ are fixed and unknown. In general, $g$ and $f$ are non-linear. Here, $X$ is the input random variable and $E$ is the environment random variable that is independent of $X$. We say that $X\in \mathbb{R}^{d_x}$ causes $Y\in \mathbb{R}^{d_y}$ if they satisfy a generative process, such as Eq.~\ref{eq:causal}.

We present methods for inferring whether $X$ causes $Y$ (denoted by $X \to Y$) or $Y$ causes $X$, or neither.  The algorithm is provided with i.i.d samples $\{(x_i,y_i)\}^{m}_{i=1} \sim \mathbb{P}^m_{X,Y}$ (the distribution of $m$ i.i.d samples from the joint distribution $\mathbb{P}_{X,Y}$) from the generative process of Eq.~\ref{eq:causal}. In general, by (cf. Prop~4.8, ~\citep{PetJanSch17}), for any joint distribution $\mathbb{P}_{X,Y}$ of two random variables $X$ and $Y$, there is an SCM, $Y = g(f(X),E)$, where $E$ is a noise variable, such that, $X \indep E$ and $f,g$ are some (measurable) functions. Therefore, in general, deciding whether $X$ causes $Y$ or vice versa is ill-posed when only provided with samples from the joint distribution. However,~\cite{pmlr-v6-zhang10a} showed for the one dimensional case (i.e., $X,Y\in \mathbb{R}$) that under reasonable conditions, a representation $Y = g(f(X)+E)$ holds only in one direction. In Sec.~\ref{sec:analysis}, we extend this theorem and show that a representation $Y=g(f(X),E)$ holds only in one direction when $g$ and $f$ are assumed to be neural networks and $X,Y$ are multidimensional (we call such SCMs neural SCMs). 

Throughout the paper, we denote by $\mathbb{P}_{U}[u] := \mathbb{P}[U \leq u]$ the cumulative distribution function of a uni/multi-variate real valued random variable $U$ and $\mathbb{P}$ is a standard Lebesgue measure. Additionally, we denote by $p_U(u) = \frac{\textnormal{d}}{\textnormal{d}u} \mathbb{P}_{U}[u]$ the probability density function of $U$ (if exists, i.e., $\mathbb{P}_{U}[u]$ is absolutely continuous). We denote by $\mathbb{E}_{u \sim U}[f(u)]$ the expected value of $f(u)$ for $u$ that is distributed by $\mathbb{P}_U[u]$. The identity matrix of dimension $n\times n$ is denoted by $I_n$ or $I$, when the dimension is obvious from the context. 

\subsection{Related Work}

In causal inference, the algorithm is provided with a dataset of matched samples $(x,y)$ of two random variables $X$ and $Y$ and decides whether $X$ causes $Y$ or vice versa. The early wisdom in this area asserted that this asymmetry of the data generating process (i.e., that $Y$ is computed from $X$ and not vice versa) is not apparent from looking at $\mathbb{P}_{X,Y}$ alone. That is, in general, provided with samples from the joint distribution $\mathbb{P}_{X,Y}$ of two variables $X,Y$ does tell us whether it has been induced by an SCM from $X$ to $Y$ or from $Y$ to $X$.

In publications, such as~\citep{pearl,Spirtes2000}, it is argued that in order to decide whether $X$ causes $Y$ or vice versa, one needs to observe the influence of interventions on the environment parameter. To avoid employing interventions, most publications assume prior knowledge on the generating process and/or independence between the cause and the mechanism.

Various methods for causal inference under the SCM have been suggested. Many of these methods are based on independence testing, where the algorithm models the data as $Y = g(f(X),E)$ (and vice versa) and decides upon the side that provides a better fitting in terms of mapping accuracy and independence between $f(X)$ and $E = r(X,Y)$. The LiNGAM~\citep{Shimizu:2006:LNA:1248547.1248619} algorithm assumes that the SCM takes the form $Y = \beta X + E$, where $X \indep E$, $\beta \in \mathbb{R}$ and $E$ is non-Gaussian. The algorithm learns $\beta$, such that, $X$ and $Y-\beta X$ are independent by applying independent component analysis (ICA). The Direct-LiNGAM~\citep{Shimizu:2011:DDM:1953048.2021040} extends this method and replaces the mutual information minimization with a non-parametric kernel based loss~\citep{Bach:2003:KIC:944919.944920}.  However, the computation of this loss is of order $\Theta(m^2)$ in the the worst case ($m$ is the number of samples).

The ANM approach~\citep{NIPS2008_3548} extends LiNGAM's modeling and assumes that $Y = f(X) + E$, where $X \indep E$. A Gaussian Process is employed as the learned mechanism between the two random variables. The function $f$ is trained to map between $X$ and $Y$ (and vice versa) and the method then tests whether, $X$ and $f(X)-Y$ are independent. The independence test is based on kernels~\citep{Gretton:2005:KMM:1046920.1194914}.

A different extension of LiNGAM is the PNL algorithm by ~\cite{pmlr-v6-zhang10a}. This algorithm learns a mapping between $X$ and $Y$ (and vice versa) of the form $Y = g(f(X) + E)$, where $f(X)$ and $E$ are restricted to be independent. To do so, PNL trains two neural networks $f$ and $g$ to minimize the mutual information between $f(X)$ and $E = g^{-1}(Y) - f(X)$. The main disadvantage of this method is the reliance on the minimization of the mutual information. It is often hard to measure and optimize the mutual information directly, especially in higher dimensions. In many cases, it requires having an explicit modeling of the density functions, because of the computation of expected log-probability within the formulation of the entropy measure. 

In our multivariate method, we take a similar approach to the above methods. However, our GAN-based independence constraint is non-parametric, is applied on the observations rather on an explicit modeling of the density functions, and the method is computationally efficient. In addition, we do not assume restrictive structural assumptions and treat the generic case, where the effect is of the form $Y = g(f(X),E)$. 

Another independence constraint is applied by the Information Geometric Causal Inference (IGCI)~\citep{DBLP:journals/corr/abs-1203-3475} approach, which determines the causal relationship in a deterministic setting $Y = f(X)$ under an independence assumption between the cause $X$ and the mechanism $f$, $\Cov(\log f'(x),p_X) = 0$. 

The Conditional Distribution Similarity Statistic (CDS)~\citep{Fonollosa2016ConditionalDV} measures the standard deviation of the values of $Y$ (resp. $X$) after binning in the $X$ (resp. $Y$) direction. The lower the standard deviation, the more likely the pair to be $X \to Y$. The CURE algorithm~\citep{pmlr-v38-sgouritsa15} compares between $X \to Y$ and $Y \to X$ directions in the following manner: if we can estimate $p_{X|Y}$ based on samples from $p_Y$ more accurately than $p_{Y|X}$ based on samples from $p_X$, then $X \to Y$ is inferred.

The BivariateFit method learns a Gaussian Process regressor in both directions and decides upon the side that had the lowest error. The RECI method~\citep{pmlr-v84-bloebaum18a} trains a regression model (a logistic function, polynomial functions, support vector regression, or a neural networks) in both directions, and returns the side that produced a lower MSE loss. The CGNN algorithm~\citep{goudet:hal-01649153} uses the Maximum Mean Discrepancy (MMD) distance between the distribution produced by modeling $Y$ as an effect of $X$, $(X,g(X,E))$ (and vice versa), and the ground truth distribution. The algorithm compares the two distances and returns the direction that led to a smaller distance. The Gaussian Process Inference model (GPI)~\citep{NIPS2010_4173}  builds two generative models, one for $X \to Y$ and one for $Y \to X$. The distribution of the candidate cause variable is modelled as a Gaussian Mixture Model, and the mechanism $f$ is a Gaussian Process. The causal direction is determined from the generative model that best fits the data.

Finally, it is worth mentioning that several other methods, such as~\citep{HeinzeDeml2017InvariantCP,Zhang:2011:KCI:3020548.3020641} assume a different type of SCM, where the algorithm is provided with separate datasets that correspond to different environments, i.e., sampled i.i.d from $\mathbb{P}_{X,Y|E}$, where the value of $E$ is fixed for all samples in the dataset. In these publications, a different independence condition is assumed: $Y$ is independent of $E$ given $X$. This assumption fails in our setting, since we focus on the vanilla SCM, where the algorithm is  provided only with observational i.i.d. samples of $X$ and $Y = g(f(X),E)$ and the samples are not divided into subsets that are invariant w.r.t $E$. 

\section{The Univariate Case}\label{sec:onemethod}

In this section, we show that the univariate SCM does not necessarily capture causality. For this purpose, we describe a method for identifying the cause and the effect, which considers each of the two variables independently without considering the mapping between them. The success of this method, despite neglecting any interaction between the variables, indicates that univariate SCM challenges can be solved without considering causality.

The proposed method computes a complexity score for $X$ and, independently, for $Y$. It then compares the scores and decides that the cause is the random variable with the larger score among them. Capturing the complexity of a univariate random variable without being able to anchor the observations in additional features is challenging. One can observe the probability distribution function and compute, for example, its entropy. As we show empirically, in Sec.~\ref{sec:exp}, this is ineffective. 

Our complexity scoring method, therefore, has a few stages. As a first step, it converts the random variable at hand (say, $X$) into a multivariate random variable. This is done by sorting the samples of the random variable, and then cutting the obtained list into fixed sized vectors of length $k$. We discard the largest measurements in the case, where the number of samples is not a multiple of $k$. We denote the random variable obtained this way by $U$. At the second stage, the method computes the complexity of the obtained random variable $U$ using an autoencoder reconstruction error. 

\subsection{Reconstruction Errors as Complexity Measures}

One can consider the complexity of a multivariate random variable in various ways. We consider non-negative complexity measures $C(X)$, which satisfy the weak assumption that when $X$ and $Y$ are independent then their complexities are lower than the complexity of their concatenation:
\begin{equation}\label{eq:additive}
C(X,Y) \geq \max(C(X),C(Y)).
\end{equation}
Examples of sample complexity measures that satisfy this condition are the Shannon Entropy and the Kolmogorov Complexity. The following lemma shows that a complexity that is based on autoencoder modeling is also in this family.

Let $\mathcal{F} = \{\mathcal{H}^d\}^{\infty}_{d=1}$ be a family of classes of autoencoders $A:\mathbb{R}^d \to \mathbb{R}^d$. Assume that the family $\mathcal{F}$ is closed to fixations, i.e., for any autoencoder $A \in \mathcal{H}^{d_1+d_2}$ and any fixed vector $y^* \in \mathbb{R}^{d_2}$ ($x^* \in \mathbb{R}^{d_1}$), we have: $A(x,y^*)_{1:d_1} \in \mathcal{H}^{d_1}$ ($A(x^*,y)_{d_1+1:d_2} \in \mathcal{H}^{d_2}$). Here, $v_{i:j} = (v_i,\dots,v_j)$. Note that this is the typical situation when considering neural networks with biases. 

Let $X$ be a random variable. Let $X$ be a multivariate random variable dimension $d$. We define the autoencoding complexity of $X$ as follows:
\begin{equation}
C_{\mathcal{F}}(X) := \min_{A^* \in \mathbb{H}^d}
\mathbb{E}_{x \sim X}\left[\ell(A^*(x),x) \right]
\end{equation}
where $\ell(a,b)$ is some loss function.  

\begin{restatable}{lemma}{compautoenc}\label{lem:compAuto}
Let $\{\mathcal{H}^d\}^{\infty}_{d=1}$ be a family of classes of autoencoders that is closed to fixations. The function $C_{\mathcal{F}}(X)$ is a proper complexity measure.
\end{restatable}

\subsection{The AEQ method}

The AEQ method we propose estimates and compares the auto-encoder reconstruction error of the quantile vectors of $X$ and $Y$. It is important to note that it does not imply that the AEQ method compares between the entropies of $X$ and $Y$.

Once the random variable $U$ is obtained as the quantiles of a random variable (either $X$ or $Y$), our method trains an autoencoder $A:\mathbb{R}^k \to \mathbb{R}^k$ on $U$. $A$ is trained to minimize the following objective:
\begin{equation}\label{eq:objectiveEnc}
\mathcal{L}_{\textnormal{recon}}(A) := \mathbb{E}_{u\sim U} [\ell(A(u), u)] 
\end{equation}
where $\ell(a,b)$ is some loss function. In our implementation, we employ the $L_2$-loss function, defined as $\ell(a,b)  = \|a-b\|^2_2$.  Finally, the method uses the value of $\mathcal{L}_{\textnormal{recon}}(A)$, which we refer to as the AEQ score, as a proxy for the complexity of $X$ (smaller loss means lower complexity). It decides that $X$ or $Y$ is the cause, based on which side provides a higher AEQ.

As we show in Sec.~\ref{sec:exp}, the proposed causality-free method is as successful at solving SCM challenges as the leading literature methods. However, we do not propose it as a standalone method, and rather develop it to show the shortcoming of the univariate SCM setting and the associated literature datasets.

\section{The Multivariate Case}\label{sec:multi}

For the univariate case, one can consider the complexity of the $X$ and $Y$ variables of the SCM and infer directionality. We propose the AEQ complexity for this case, since more conventional  complexities are ill-defined for unordered 1D data or, in the case of entropy, found to be ineffective.

The following technical lemma shows that for any complexity measure $C$, one cannot infer directionality in the multivariate SCM based on $C$. 

\begin{restatable}{lemma}{complexityMulti} Let $C$ be a complexity measure of multivariate random variables (i.e, non-negative and satisfies Eq.~\ref{eq:additive}). Then, there are triplets of random variables $(X,E,Y)$ and $(\hat{X},E,Y)$ and functions $g$ and $g'$, such that, $Y = g(X,E)$, $Y = g'(\hat{X},E)$, $C(X) < C(Y)$ and $C(\hat{X}) > C(Y)$. Therefore, $C$ cannot serve as a score for causal inference.
\end{restatable}

We now turn our attention to a new multivariate causality inference method.

\subsection{An Adversarial Method for Causal Inference}\label{sec:methodCI} 

Our causality inference algorithm trains neural networks $G,F,R$ and $D$. The success of fitting these networks serves as the score for the causality test. The function $F$ models the function $f$, $G$ models $g$ and $R(Y)$ aims to model the environment parameter $E$. In general, our method aims at solving the following objective:
\begin{equation}\label{eq:objectiveMulti}
\begin{aligned}
&\min_{G,F,R} \mathcal{L}_{\textnormal{err}}(G,F,R) := \frac{1}{m}\sum^{m}_{i=1}\|  G(F(a_i),R(b_i)) - b_i \|^2_2\\
&\;\;\textnormal{s.t: } A \indep R(B) 
\end{aligned}
\end{equation}
where $A$ is either $X$ or $Y$ and $B$ is the other option, and $a_i=x_i, b_i=y_i$ or $a_i=y_i, b_i=x_i$ accordingly. To decide whether $X \to Y$ or vice versa, we train a different triplet $G,F,R$ for each direction and see if we can minimize the mapping error $\mathcal{L}_{\textnormal{err}}$ subject to independence. We decide upon a specified direction, if the loss can be minimized subject to independence. In general, searching within the space of functions that satisfy $A \indep R(B)$ is an intractable problem. However, we can replace it with a loss term that is minimized when $A \indep R(B)$.

{\bf Independence loss\quad} We would like $R(B)$ to capture the information encoded in $E$. Therefore, restrict $R(B)$ and $A$ to be independent in each other. We propose an adversarial loss for this purpose, which is a modified version of a loss proposed by \cite{brakel} and later analyzed by \citep{press2018emerging}.

This loss measures the discrepancy between the joint distribution $\mathbb{P}_{A,R(B)}$ and the product of the marginal distributions $\mathbb{P}_{A} \times \mathbb{P}_{R(B)}$. 
Let $d_F$ ($d_R$) be the dimension of $F$'s output ($R$). To measure the discrepancy, we make use of a discriminator $D:\mathbb{R}^{d_a+d_R} \to [0,1]$ ($d_a$ equals $d_x$ or $d_y$ depending on $A=X$ or $A=Y$) that minimizes the following term:
\begin{equation}
\begin{aligned}
\mathcal{L}_{D}(D;R) :=&\frac{1}{m}\sum^{m}_{i=1}\ell(D(a_i,R(b_i)),1) \\
&+ \frac{1}{m}\sum^{m}_{i=1}\ell(D(\hat{a}_i,R(\hat{b}_i)),0)
\end{aligned}
\end{equation}
where $D$ is a discriminator network, and  $l(p,q) = -(q\log(p)+(1-q)\log(1-p))$ is the binary cross entropy loss for $p\in\left[0,1\right]$ and $q\in\{0, 1\}$. In addition, $\{(\hat{a}_i,\hat{b}_i)\}^{m}_{i=1}$ are i.i.d samples from $\mathbb{P}_{A} \times \mathbb{P}_{B}$. To create these samples, we sample independently $\hat{a}_i$ and $\hat{b}_i$ from the respective training sets $\{(\hat{a}_i\}^{m}_{i=1}$ and $\{(\hat{b}_i\}^{m}_{i=1}$ and then arbitrarily match them into couples $(\hat{a}_i,\hat{b}_i)$. 

To restrict that $R(B)$ and $A$ are independent, $R$ is trained to confuse the discriminator $D$ such that the two sets of samples are indistinguishable by $D$, 
\begin{equation}
\begin{aligned}
\mathcal{L}_{\textnormal{indep}}(R;D)
:= &\frac{1}{m}\sum^{m}_{i=1}\ell(D(a_i,R(b_i)),1) \\
&+ \frac{1}{m}\sum^{m}_{i=1}\ell(D(\hat{a}_i,R(\hat{b}_i)),1)
\end{aligned}
\end{equation}

{\bf Full objective\quad} The full objective of our method is then translated into the following program:
\begin{equation}\label{eq:fullobjective}
\begin{aligned}
&\min_{G,F,R}\; \mathcal{L}_{\textnormal{err}}(G,F,R) + \lambda \cdot \mathcal{L}_{\textnormal{indep}}(R;D) \\
&\min_{D}\; \mathcal{L}_D(D;R)
\end{aligned}
\end{equation}
Where $\lambda$ is some positive constant. The discriminator $D$ minimizes the loss $\mathcal{L}_D(D;R)$ concurrently with the other networks. Our method decides if $X$ causes $Y$ or vice versa, by comparing the score $\mathcal{L}_{\textnormal{err}}(G,F,R)$. A lower error means a better fit. The full description of the architecture employed for the encoders, generator and discriminator is given in Appendix~\ref{app:arch}. A sensitivity experiment for the parameter $\lambda$ is provided in Appendix~\ref{app:sensitive}. 

In addition to the success in fitting, we also measure the degree of independence between $A$ and $R(B)$. We denote by $c_{\textnormal{real}}$ the percentage of samples $(a_i,b_i)$ that the discriminator classifies as $1$ and by $c_{\textnormal{fake}}$ the percentage of samples $(\hat{a}_i,\hat{b}_i)$ that are classified as $0$. We note that when $c_{\textnormal{real}} \approx 1- c_{\textnormal{fake}}$, the discriminator is unable to discriminate between the two distributions, i.e., it is wrong in classifying half of the samples. We, therefore, use $|c_{\textnormal{real}} + c_{\textnormal{fake}} -1| $ as a measure of independence. 

\subsection{Analysis}\label{sec:analysis}

In this section, we analyze the proposed method. In Thm.~\ref{thm:identif}, we show that if $X$ and $Y$ admit a SCM in one direction, then it admits a SCM in the opposite direction, only if the involved functions satisfy a specific partial differential equation.

\begin{restatable}[Identifiability of neural SCMs]{theorem}{identif}\label{thm:identif} Let $\mathbb{P}_{X,Y}$ admit a neural SCM from $X$ to $Y$ as in Eq.~\ref{eq:causal}, such that $p_X$, and the activation functions of $f$ and $g$ are three-times differentiable. Then it admits a neural SCM from $Y$ to $X$, only if $p_X$, $f$, $g$ satisfy Eq.~\ref{eq:diffeq} in the appendix.
\end{restatable}

This result generalizes the one-dimensional case presented in~\citep{pmlr-v6-zhang10a}, where a one-dimensional version of this differential equation is shown to hold in the analog case.

In the following theorem, we show that minimizing the proposed losses is sufficient to recover the different components, i.e., $F(X) \propto f(X)$ and $R(Y) \propto E$, where $A \propto B$ means that $A=f(B)$ for some invertible function $f$. 

\begin{restatable}[Uniqueness of Representation]{theorem}{sufficient}\label{thm:sufficient} Let $\mathbb{P}_{X,Y}$ admit a nonlinear model from $X$ to $Y$ as in Eq.~\ref{eq:causal}, i.e., $Y = g(f(X),E)$ for some random variable $E \indep X$. Assume that $f$ and $g$ are invertible. Let $G$, $F$ and $R$ be functions, such that, $\mathcal{L}_{\textnormal{err}} := \mathbb{E}_{(x,y) \sim (X,Y)}[\|G(F(x),R(y))-y\|^2_2] = 0$ and $G$ and $F$ are invertible functions and $X \indep R(Y)$. Then, $F(X) \propto f(X)$ and $R(Y) \propto E$.
\end{restatable}

where, $\mathcal{L}_{\textnormal{err}}$ is the mapping error proposed in Eq.~\ref{eq:objectiveMulti}. In addition, the assumption $X \indep R(Y)$ is sufficed by the independence loss. 

A more general results, but which requires additional terminology, is stated as Thm.~\ref{thm:sufficientExt} in Appendix~\ref{app:analysis}. It extends Thm.~\ref{thm:sufficient} to the case, where the mapping loss is not necessarily zero and the independence $X \indep R(Y)$ is replaced by a discriminator-based independence measure. Thm.~\ref{thm:sufficientExt} also gets rid of the assumption that the various mappings $f,g$ and $F, G$ are invertible. In this case, instead of showing that $R(Y) \propto E$, we provide an upper bound on the reconstruction of $E$ out of $R(Y)$ (and vice versa) that improves as the training loss of $G$, $F$ and $R$ decreases.

To conclude our analysis, by Thm.~\ref{thm:identif}, under reasonable assumptions, if $X$ and $Y$ admit a multivariate SCM in direction $X \to Y$, then, there is no such representation in the other direction. By Thm.~\ref{thm:sufficient}, by training our method in both directions, one is able to capture the causal model in the correct direction. This is something that is impossible to do in the other direction by Thm.~\ref{thm:identif}. 

\section{Experiments}\label{sec:exp}

This section is divided into two parts. The first is devoted to showing that causal inference in the one-dimensional case highly depends on the complexities of the distributions of $X$ and $Y$. In the second part of this section, we show that our multivariate causal inference method outperforms existing baselines. Most of the baseline implementations were taken from the Causality Discovery Toolbox of \cite{kalainathan2019causal}. The experiments with PNL~\citep{pmlr-v6-zhang10a}, LiNGAM~\citep{Shimizu:2006:LNA:1248547.1248619} and GPI~\citep{NIPS2010_4173} are based on their original matlab code.

\subsection{One-Dimensional Data}

We compared the autoencoder method on several well-known one dimensional cause-effect pairs datasets. Each dataset consists of a list of pairs of real valued random variables $(X,Y)$ with their direction $1$ or $0$, depending on $X \to Y$ or $Y \to X$ (resp.). For each pair, we have a dataset of samples $\{(x_i,y_i)\}^{m}_{i=1}$. 

Five cause-effect inference datasets, covering a wide range of associations, are used. CE-Net~\citep{goudet:hal-01649153} contains 300 artificial cause-effect pairs generated using random distributions as causes, and neural networks as causal mechanisms. CE-Gauss contains 300 artificial cause-effect pairs as
generated by~\cite{JMLR:v17:14-518}, using random mixtures of Gaussians as causes, and Gaussian Process priors as causal
mechanisms. CE-Multi~\citep{goudet:hal-01649153} contains 300 artificial cause-effect pairs built with random linear and polynomial causal mechanisms. In this dataset, simulated additive or multiplicative noise is applied before or after the causal mechanism. 

The real-world datasets include the diabetes dataset by~\cite{FrankAsuncion2010}, where causality is from $\textnormal{Insulin} \to \textnormal{Glucose}$. Glucose curves and Insulin doses were analysed for 69 patients, each serves as a separate dataset. To match the literature protocols, the pairs are taken in an orderless manner, ignoring the time series aspect of the problem. Finally, the T\"ubingen cause-effect pairs dataset by~\cite{JMLR:v17:14-518} is employed. This dataset is a collection of 100 heterogeneous, hand-collected, real-world cause-effect samples. 

The autoencoder $A$ employed in our method, Eq.~\ref{eq:objectiveEnc}, is a fully-connected five-layered neural network with three layers for the encoder and two layers for the decoder. The hyperparameters of this algorithm are the sizes of each layer, the activation function and the input dimension, i.e., length of sorted cuts (denoted by $k$ in Sec.~\ref{sec:onemethod}). Throughout the experiments, we noticed that the hyperparameter with the highest influence is the input dimension. For all datasets, results are stable in the range of $200 \leq k \leq 300$, and we, therefore, use $k=250$ throughout the experiments. For all datasets, we employed the ReLU activation function, except the T\"ubingen dataset, where the sigmoid activation function produced better results (results are also reasonable with ReLU, but not state of the art).

In addition to our method, we also present results obtained with the entropy of each individual variable as a complexity measure. This is done by binning the values of the variables into 50 bins. Other numbers of bins produce similar results.

Tab.~\ref{tab:auc} presents the mean AUC for each literature benchmark. As can be seen, the AEQ complexity measure produces reasonable results in comparison to the state of the art methods, indicating that the 1D SCM  can be overcome by comparing per-variable scores. On the popular T\"ubingen dataset, the AEQ computation outperforms all literature methods. 

Tab.~\ref{tab:quantitative_results} presents accuracy rates for various methods on the T\"ubingen dataset, where such results are often reported in the literature. As can be seen, our interaction-less method outperforms almost all other methods, including methods that employ supervised learning of the cause-effect relation.

\begin{table}
\caption{Mean AUC rates of various baselines on different one dimensional cause-effect pairs datasets. Our interaction-less AEQ algorithm achieves competitive results on most datasets.}
\label{tab:auc}
\begin{center}
 \begin{tabular}{@{}l@{~}c@{~}c@{~}c@{~}c@{~}c@{}} 
\toprule
 & CE- & CE- & CE- & T\"ubi- & Dia-\\ 
Method & Net & Gauss & Multi & ngen & betes \\ \midrule
\midrule
BivariateFit & 77.6 & 36.3 & 55.4 & 58.4 & 0.0 \\ 
 LiNGAM\citep{Shimizu:2006:LNA:1248547.1248619} & 43.7 & 66.5 & 59.3 & 39.7 & {\bf 100.0} \\ 
 CDS~\citep{Fonollosa2016ConditionalDV} & 89.5 & 84.3 & 37.2 & 59.8 & 12.0 \\ 
 IGCI~\citep{DBLP:journals/corr/abs-1203-3475} & 57.2 & 33.2 & 80.7 & 62.2 & {\bf 100.0} \\ 
 ANM~\citep{NIPS2008_3548} & 85.1 & 88.9 & 35.5 & 53.7 & 22.2 \\ 
 PNL\citep{pmlr-v6-zhang10a} & 75.5 & 83.0 & 49.0 & 68.1 & 28.1 \\ 
 GPI~\citep{NIPS2010_4173} & 88.4 & {\bf 89.1} & 65.8 & 66.4 & 92.9 \\ 
 RECI~\citep{pmlr-v84-bloebaum18a} & 60.0 & 64.2 & 85.3 & 62.6 & 95.4 \\ 
 CGNN~\citep{goudet:hal-01649153} & {\bf 89.6} &  82.9 & {\bf 96.6} & 79.8 & 34.1 \\ 
 Entropy as complexity & 49.6 & 49.7 & 50.8 & 54.5 & 53.4 \\ 
Our AEQ comparison & 62.5 & 71.0 & 96.0 & {\bf 82.8} & 95.0 \\  [1ex] 
 \bottomrule
\end{tabular}
\end{center}
\end{table}

\begin{table}[t]%
\caption{Accuracy rates of various baselines on the CE-T\"ub dataset. Our interaction-less algorithm AEQ achieves almost SOTA accuracy.}
\label{tab:quantitative_results}
\begin{minipage}{\columnwidth}
\begin{center}
 \begin{tabular}{l c c c c c} 
\toprule
Method & Supervised & Acc  \\ \midrule
LiNGAM~\citep{Shimizu:2006:LNA:1248547.1248619} & - & 44.3\%  \\ 
BivariateFit & - & 44.9\% \\ 
Entropy as a complexity measure & - & 52.5\% \\ 
IGCI~\citep{DBLP:journals/corr/abs-1203-3475} & -& 62.6\%  \\
CDS~\citep{Fonollosa2016ConditionalDV} & -& 65.5\%  \\ 
ANM~\citep{NIPS2008_3548} & -& 59.5\%  \\ 
CURE~\citep{pmlr-v38-sgouritsa15} & -& 60.0\%\footnote{The accuracy of CURE is reported on version 0.8 of the dataset in~\citep{pmlr-v38-sgouritsa15} as 75\%. In~\cite{pmlr-v84-bloebaum18a} they re-ran this algorithm and achieved an accuracy rate of around 60\%.}\\
GPI~\citep{NIPS2010_4173} & -&62.6\% \\ 
PNL~\citep{pmlr-v6-zhang10a} & -& 66.2\%  \\ 
CGNN~\citep{goudet:hal-01649153} & -& 74.4\%  \\
RECI~\citep{pmlr-v84-bloebaum18a} & -& 77.5\%  \\ 
SLOPE~\citep{marx} & -& {\bf 81.0}\% \\
Our AEQ comparison & -& 80.0\% \\  
\midrule
Jarfo~\citep{Fonollosa2016ConditionalDV} & + & 59.5\% \\
RCC~\citep{towards} & +& 75.0\%\footnote{The accuracy scores reported in~\citep{towards} are for version 0.8 of the dataset, in~\citep{ncc} they re-ran RCC~\citep{towards} on version 1.0 of the dataset.} \\
NCC~\citep{ncc} & +& 79.0\% \\
\bottomrule
\end{tabular}
\end{center}
\end{minipage}
\end{table}%

\subsection{Multivariate Data}

\begin{table}[t]%
\caption{Mean AUC rates of various baselines on different multivariate cause-effect pairs datasets. The datasets are designed and balanced, such that an autoencoder method would fail. Our method achieves SOTA results.}
\label{tab:multiv}
\begin{center}
 \begin{tabular}{@{}l@{~}c@{~}c@{~}c@{~}c@{~}c@{}} 
\toprule
Method & MCE- & MCE- & MCE- & MOUS- \\
& Poly & Net & SigMix & MEG \\  \midrule
AE reconstruction & 57.2 & 42.4 & 22.3 & 41.2 \\
\midrule
BivariateFit & 54.7 & 48.4 & 48.2 
& 44.2\\
IGCI~\citep{DBLP:journals/corr/abs-1203-3475} & 41.9 & 49.3 & 59.8 & 56.0\\
CDS~\citep{Fonollosa2016ConditionalDV} & 63.8 & 57.0 & 62.1 & 89.9\\
ANM~\citep{NIPS2008_3548} & 52.2 & 51.1 & 46.4 
& 52.4 \\ 
PNL~\citep{pmlr-v6-zhang10a} & 76.4 & 54.7 & 16.8  
& 56.3 \\ 
CGNN~\citep{goudet:hal-01649153} & 47.8  & 67.8  & 58.8  
&  40.9  \\ 
Our method & {\bf 95.3} & {\bf 84.2} & {\bf 98.5} 
& {\bf 97.7} \\
 \bottomrule
\end{tabular}
\end{center}
\end{table}%

\begin{table}
\caption{Results of various methods on different variations of the MOUS-MEG dataset. R stands for the MEG scan at rest, W stands for the word presented to the subject and A stands for the MEG scan, when the subject is active.}
\label{tab:meg}
\begin{center}
 \begin{tabular}{@{}l@{~~}c@{~~}c@{~~}c@{~~}c@{~~}c@{}} 
\toprule
Method & R + W$\to$ A & R $\to$ A  & W $\to$ A  
\\ \midrule
Expected to be causal & Yes & No & No\\
\midrule
AE reconstruction & 41.2 & 51.7 & 98.6 
\\
\midrule
BivariateFit & 44.2 & 58.1 & 0.0 
\\
IGCI~\citep{DBLP:journals/corr/abs-1203-3475} & 56.0 & 50.6 & 42.2   
\\
CDS~\citep{Fonollosa2016ConditionalDV} & 89.9  & 52.1 & 90.2
\\
ANM~\citep{NIPS2008_3548} & 52.4 & 49.3 & 0.0 
\\
PNL~\citep{pmlr-v6-zhang10a} & 56.3 & 43.7 &  0.0 
\\
CGNN~\citep{goudet:hal-01649153} & 40.9  &  52.2 & 100.0
\\
Our method & 97.7  & 44.4 & 0.0 
\\  [1ex] 
 \bottomrule
\end{tabular}
\end{center}
\end{table}%

\begin{table}[t]%
\caption{Emergence of independence. Ind C (Ind E) is the mean of $\vert c_{\textnormal{real}} + c_{\textnormal{fake}} - 1 \vert$ over all pairs of random variables, epochs and samples, when training the method from $X$ to $Y$ (vice versa). w/o backprop means without backpropagating gradients from $D$ to $R$.}
\label{tab:abl}

\begin{center}
\begin{tabular}{@{}c@{~~}c@{~~}c@{~~}c@{~~}c@{~~}c@{~~}c@{}}
\hline
\multicolumn{1}{l}{\multirow{2}{*}{}} & \multicolumn{3}{c}{Full method} & \multicolumn{3}{c}{w/o backprop} \\ 
\cmidrule(lr){2-4} 
\cmidrule(lr){5-7} 
\multicolumn{1}{c}{Dataset} & AUC  & Ind C & Ind E & AUC & Ind C & Ind E     \\ 
\cline{1-7} 
\multicolumn{1}{c}{MCE-Poly}         & 95.3 & 0.06 & \multicolumn{1}{c}{0.05} & 95.1     & 0.10     & 0.10     \\
\multicolumn{1}{c}{MCE-Net}          & 84.2 & 0.28  & \multicolumn{1}{c}{0.31}  & 65.1     & 0.55     & 0.55     \\
\multicolumn{1}{c}{MCE-SigMix}  & 98.5 & 0.05 & \multicolumn{1}{c}{0.06} & 98.8     & 0.16      & 0.20     \\
\multicolumn{1}{c} {MOUS-MEG}  &  97.7  & 0.14 & \multicolumn{1}{c}{0.14} & {80.7}  & {0.74} & {0.75} \\ \hline
\end{tabular}
\end{center}
\end{table}%

We first compare our method on several synthetic datasets. Each dataset consists of a list of pairs of real multivariate random variables $(X,Y)$ with their direction $1$ or $0$, depending on $X \to Y$ or $Y \to X$ (resp.). For each pair, we have a dataset of samples $\{(x_i,y_i)\}^{m}_{i=1}$. 

We employ five datasets, covering multiple associations. Each dataset contains 300 artificial cause-effect pairs. The cause random variable is of the form $X = h(z)$, where $h$ is some function and $z \sim \mathcal{N}(0,\sigma^2_1\cdot I_n)$. The effect is of the form $Y = g(u(X,E))$, where $E \sim \mathcal{N}(0,\sigma^2_2 \cdot I_{n})$ is independent of $X$, $u$ is a fixed function that combined the cause $X$ and the noise term $E$ and $g$ is the causal mechanism. For each dataset, the functions $h$ and $g$ are taken from the same family of causal mechanisms $\mathcal{H}$. Each pair of random variables is specified by randomly selected functions $h$ and $g$. 

The synthetic datasets extend the standard synthetic data generators of~\cite{kalainathan2019causal} to the multivariate causal pairs. 
MCE-Poly is generated element-wise polynomials composed on linear transformations as mechanisms and $u(X,E)=X+E$. MCE-Net pairs are generated using neural networks as causal mechanisms and $u$ is the concatenation operator. The mechanism in MCE-SigMix consists of linear transformation followed by element wise application of $q_{a,b,c}(x) := ab(\Tilde{x} + c) / (1 + |b \cdot (\Tilde{x} + c)|)$, where $a,b,c$ are random real valued numbers, which are sampled for each pair and $\Tilde{x} = x + e$, where $e$ is the environment random variable. In this case, $u(X,E) = X+E$. 
We noticed that a-priori, the produced datasets are imbalanced in a way that the reconstruction error of a standard autoencoder on each random variable can be employed as a score that predicts the cause variable with a high accuracy. Therefore, in order to create balanced datasets, we varied the amount of noise dimensions and their intensity, until the autoencoder reconstruction error of both $X$ and $Y$ became similar. Note that for these multivariate variables, we do not use quantiles and use the variables themsevles. As the AutoEncoder reconstruction results in Tab.~\ref{tab:multiv} show, in the MCE-SigMix dataset, balancing was only partly successful.

We compare our results to two types of baseline methods: (i) BivariateFit and ANM~\citep{NIPS2008_3548} are methods that were designed (also) for the multivariate case, (ii) CGNN~\citep{goudet:hal-01649153} and PNL~\citep{pmlr-v6-zhang10a} are naturally extended to this case. To extend the CDS~\citep{Fonollosa2016ConditionalDV} and IGCI~\citep{DBLP:journals/corr/abs-1203-3475} methods to higher dimension, we applied quantizations over the data samples, i.e., cluster the samples $\{x_i\}^{m}_{i=1}$ and $\{y_i\}^{m}_{i=1}$ using two distinct k-means with $k=10$, and then, each sample is replaced with its corresponding cluster to obtain a univariate representation of the data. After pre-processing the data, we apply the corresponding method. To select the hyperparameter $k$, we varied its value between $5$ to $500$ for different scales and found $10$ to provide the best results. RECI~\citep{pmlr-v84-bloebaum18a} could be extended. However, RECI's runtime is of order $\mathcal{O}(n^3)$, where $n$ is the input dimension. Other methods cannot be extended, or require significant modifications. For example, the SLOPE method~\citep{marx} heavily relies on the ability to order the samples of the random variables $X$ and $Y$. However, it is impossible to do so in the multivariate case. We could not find any open source implementation of the CURE algorithm~\citep{pmlr-v38-sgouritsa15}.

The results, given in Tab.~\ref{tab:multiv} show a clear advantage over the literature methods across the four datasets. The exact same architecture is used thorughout all experiments, with the same $\lambda$ parameter. See Sec.~1 of the supplementary material. A sensitivity analysis (see supplementary Sec.~2) shows that our results are better than all baseline methods, regardless of the parameter $\lambda$.

In addition to the synthetic datasets, we also employ the MOUS-MEG real world dataset, provided to us by the authors of~\citep{anonymous2020measuring}. This dataset is part of Mother Of Unification Studies (MOUS) dataset~\citep{schoffelen2019204}. This dataset contains magneto-encephalography (MEG) recordings of 102 healthy Dutch-speaking subjects performing a reading task (9 of them were excluded due to corrupted data). Each subject was asked to read 120 sentences in Dutch, both in the right order and randomly mixed order, which adds up to a total of over 1000 words. Each word was presented on the computer screen for 351ms on average and was separated from the next word by 3-4 seconds. Each time step consists of 301 MEG readings of the magnetometers, attached to different parts of the head. For more information see~\citep{schoffelen2019204}. For each pair $(X,Y)$, $X$ is the interval $[-1.5s,-0.5s]$ relative to the word onset concatenated with the word embedding (using the spaCy python module with the Dutch language model), this presents the subject in his ``rest'' state (i.e. the cause). $Y$ is the interval $[0,1.0s]$ relative to the word onset, which presents the subject in his ``active'' state (i.e. the effect). 

To validate the soundness of the dataset, we ran a few experiments on variations of the dataset and report the results as additional columns in Tab.~\ref{tab:meg}.
As can be seen, a dataset where the cause consists of the word embedding and the effect consists of the subject's ``active'' state is highly imbalanced. This is reasonable, since the word embedding and the MEG readings are encoded differently and are of different dimensions. In addition, when the cause is selected to be the ``rest'' state and the effect is the ``active'' state, the various algorithms are unable to infer which side is the cause and which one is the effect, since the word is missing. Finally, when considering the Rest+Word $\to$ Active variation, the relationship is expected to be causal, the AE reconstruction indicates that the dataset is balanced, and our method is the only one to achieve a high AUC rate.

{\bf Emergence of independence\quad} To check the importance of our adversarial loss in identifying the direction of causality and capturing the implicit independent representation $f(X)$ and $E$, we applied our method without training $R$ against the discriminator. Therefore, in this case, the discriminator only serves as a test whether $X$ and $R(Y)$ are independent or not and does not contribute to the training loss of $R$ ($\lambda=0$).  

As mentioned in Sec.~\ref{sec:methodCI}, the distance between $c_{\textnormal{real}} + c_{\textnormal{fake}}$ to $1$ indicates the amount of dependence between $X$ and $R(Y)$. We denote by $\textnormal{Ind C}$ the mean values of $\vert c_{\textnormal{real}} + c_{\textnormal{fake}} - 1\vert$ over all pairs of random variables and samples when training our method in the causal direction. The same mean score when training in the anti-causal direction is denoted $\textnormal{Ind E}$. As is evident from Tab.~\ref{tab:abl}, the independence is similar between the two directions, emphasizing the importance of the reconstruction error in the score.

As can be seen in Tab.~\ref{tab:abl}, the adversarial loss improves the results when there is no implicit emergence of independence. However, in cases where there is emergence of independence, the results are similar. We noticed that the values of  $\textnormal{Ind C}$ and $\textnormal{Ind E}$ are smaller for the full method. However, in MCE-Poly and MCE-SigMix they are still very small and, therefore, there is implicit emergence of independence between $X$ and $R(Y)$, even without explicitly training $R(Y)$ to be independent of $X$.

\section{Summary}

We discover an inbalance in the complexities of cause and effect in the univariate SCM and suggest a method to exploit it. Since the method does not consider the interactions between the variables, its success in predicting cause and effect indicates an inherent bias in the univariate datasets. Turning our attention to the multivariate case, where the complexity can be actively balanced, we propose a new method in which the learned networks model the underlying SCM itself. Since the noise term $E$ is unknown, we replace it by a function of $Y$ that is enforced to be independent of $X$. We also show that under reasonable conditions, the independence emerges, even without explicitly enforcing it.

\section{Acknowledgements}

This project has received funding from the European Research Council (ERC) under the European
Union’s Horizon 2020 research and innovation programme (grant ERC CoG 725974). The authors
would like to thank Dimitry Shaiderman for insightful discussions. The contribution of Tomer Galanti is part of Ph.D. thesis research conducted at Tel Aviv University.

\newpage
\bibliography{arxiv_version}
\bibliographystyle{icml2020}

\newpage

\appendix

\section{Architecture for All Multivariate Experiments}\label{app:arch}

The functions $G$, $F$, $R$ and $D$ in the adversarial multivariate method are fully connected neural networks and their architectures are as follows: $F$ is a 2-layered network with dimensions $100 \to 60 \to 50$, $R$ is a 3-layered network with dimensions $100 \to 50 \to 50 \to 20$, $G$ is a 2-layers neural network with dimensions $50+20 \to 80 \to 100$ (the input has 50 dimensions for $F(X)$ and 20 for $R(Y)$). The discriminator is a 3-layers network with dimensions $100+20 \to 60 \to 50 \to 2$ (the input is the concatenation of $X$ and $R(Y)$). The activation function in all networks is the sigmoid function except the discriminator that applies the leaky ReLU activation. For all networks, the activation is not applied at the output layer. 

Throughout the experiments the learning rate for training $G$, $F$ and $R$ is 0.01 and the learning rate of $D$ is 0.001.

\section{Sensitivity Experiment}\label{app:sensitive}

To check that our results are robust with respect to $\lambda$, we conducted a sensitivity analysis. In this experiment we ran our algorithm on the MOUS-MEG dataset (i.e., Rest + Word $\to$ Active variation) with $\lambda$ that varies between $10^{-5}$ to 1 in a logarithmic scale. As can be seen in Fig.~\ref{fig:sens_exp}, our algorithm is highly stable to the selection of $\lambda \in [10^{-5},10^{-1}]$. The performance decays (gradually) only for $\lambda \geq 0.1$. 

\begin{figure}
\begin{center}
  \includegraphics[width=1.1\linewidth]{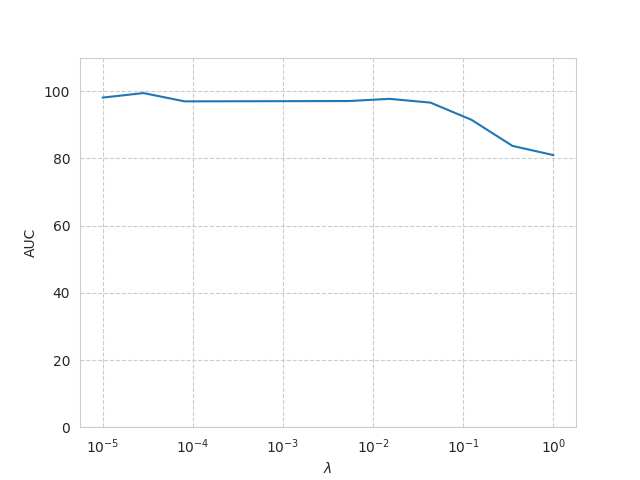}
\end{center}
  \caption{Sensitivity experiment. The graph presents the AUC of our algorithm on MOUS-MEG dataset with $\lambda$, which varies between $10^{-5}$ to 1 in a logarithmic scale.}
\label{fig:sens_exp}
\end{figure}

\section{Analysis}\label{app:analysis}

\subsection{Terminology and Notations}

We recall some relevant notations and terminology. For a vector $x = (x_1,\dots,x_n) \in \mathbb{R}^n$ we denote $\|x\|_2 := \sqrt{\sum^{n}_{i=1} x^2_i}$ the Euclidean norm of $x$. For a differentiable function $f : \mathbb{R}^{m} \rightarrow \mathbb{R}^{n}$ and $x \in \mathbb{R}^m$, we denote by 
\begin{equation}
\Diff(f(x)) := \left(\frac{\partial f_i}{\partial \zeta_j}(x)\right)_{i \in [n],j \in [m]}    
\end{equation}
the Jacobian matrix of $f$ in $x$. For a twice differentiable function $f : \mathbb{R}^{m} \rightarrow \mathbb{R}$, we denote by 
\begin{equation}
\Hess(f(x)) := \left(\frac{\partial^2 f}{\partial \zeta_i \partial \zeta_j}(x)\right)_{i,j \in [m]}    
\end{equation}
the Hessian matrix of $f$ in $x$. Additionally, for a twice differentiable function $f : \mathbb{R}^{m} \rightarrow \mathbb{R}^n$, $f(x) = (f_1(x),\dots,f_n(x))$, we denote the Hessian of $f$ by $\Hess(f(x)) := (\Hess(f_1(x)),\dots,\Hess(f_n(x)))$. For a scalar function $f:\mathbb{R}^m \to \mathbb{R}$ instead of using the Jacobian notation, the gradient notation will be employed, $\nabla (f(x)) := \Diff(f(x))$. For two positive functions $f(x)$ and $g(x)$, we denote, $f(x) \lesssim g(x)$ if there is a constant $C>0$, such that, $f(x) \leq C \cdot g(x)$.

\subsection{Proofs for the Results}\label{app:proofs}

In this section we provide the proofs of the main results in the paper.

\compautoenc*


\begin{proof}
First, since $\ell(a,b) \geq 0$ for all $a,b \in \mathbb{R}^k$, this function is non-negative. Next, we would like to show that $C_{\mathcal{F}}(X,Y) \geq \max(C_{\mathcal{F}}(X),C_{\mathcal{F}}(Y))$. Let $A^*$ be the minimizer of $\mathbb{E}_{x\sim X}\left[\ell(A^*(x),x) \right]$ within $\mathcal{H}^{d_1+d_2}$. We consider that there is a vector $y^*$, such that,
\begin{equation}
\begin{aligned}
\mathbb{E}_{(x,y) \sim (X,Y)}&\left[\ell(A(x,y),(x,y)) \right]  \\
\geq &\mathbb{E}_{y \sim Y}\mathbb{E}_{x\sim X}\left[\ell(A(x,y),(x,y))  \right] \\
\geq &\mathbb{E}_{x\sim X}\left[\ell(A(x,y^*),(x,y^*)) \right] \\
\geq &\mathbb{E}_{x\sim X}\left[\ell(A(x,y^*)_{1:d_1},x)\right]
\end{aligned}
\end{equation}
We note that $A(x,y^*)_{1:d_1} \in \mathcal{H}^{d_1}$. Therefore, \begin{equation}
\begin{aligned}
&\mathbb{E}_{(x,y)\sim (X,Y)}\left[\ell(A(x,y),(x,y)) \right] \\
&\geq \min_{A^* \in \mathcal{H}^{d_1}}\mathbb{E}_{x\sim X}\left[\ell(A^*(x),x)\right]
= C_{\mathcal{F}}(X)
\end{aligned}
\end{equation}
By similar considerations, $C_{\mathcal{F}}(X,Y)$.
\end{proof}

\complexityMulti*


\begin{proof}
Let $X$ be a random variable and $E \indep X$, such that, $Y = g(X,E)$. Assume that $C(X) < C(Y)$. Then, let $X'$ be a random variable independent of $X$, such that, $C(X') > C(Y)$. Then, according to the definition of a complexity measure, we have: $C(X,X') > C(Y)$ and we have: $Y = g'(X,X',E)$, for $g'(a,b,c) = g(a,c)$. 
\end{proof}

The following lemma is an extension of Thm.~1 in~\citep{pmlr-v6-zhang10a} to real valued random variables of dimension $>1$.
\begin{lemma}\label{lem:identif}
Assume that $(X, Y)$ can be described by both:
\begin{equation}
Y = g_1(f_1(X)+E_1), \textnormal{ s.t: } X \indep E_1 \textnormal{ and $g_1$ is invertible}
\end{equation}
and
\begin{equation}
X = g_2(f_2(Y)+E_2), \textnormal{ s.t: } Y \indep E_2 \textnormal{ and $g_2$ is invertible}
\end{equation}
Assume that $g_1$ and $g_2$ are invertible and let:
\begin{equation}
\begin{aligned}
T_1 &:= g^{-1}_1(Y) \textnormal{ and } h_1 := f_2 \circ g_1 \\ 
T_2 &:= g^{-1}_2(X) \textnormal{ and } h_2 := f_1 \circ g_2   
\end{aligned}
\end{equation}
Assume that the involved densities $p_{T_2}$, $p_{E_1}$ and nonlinear functions $f_1,g_1$ and $f_2,g_2$ are third order differentiable. We then have the following equations for all $(X,Y)$ satisfying:
\begin{equation}
\begin{aligned}
&\textnormal{H}(\eta_1(t_2))\cdot \textnormal{J}(h_1(t_1)) - \textnormal{H}(\eta_2(e_1)) \cdot \textnormal{J}(h_2(t_2)) \\
&+ \textnormal{H}(\eta_2(e_1))\cdot \textnormal{J}(h_2(t_2)) \cdot \textnormal{J}(h_1(t_1)) \cdot \textnormal{J}(h_2(t_2))\\ 
&-\nabla(\eta_2(e_1)) \cdot \textnormal{H}(h_2(t_2)) \cdot \textnormal{J}(h_1(t_1)) = 0
\end{aligned}
\end{equation}
where $\eta_1(t_2) := \log p_{T_2}(t_2)$ and $\eta_2(e_1) := \log p_{E_1}(e_1)$.
\end{lemma}

\begin{proof}
The proof is an extension of the proof of Thm.~1 in~\citep{pmlr-v6-zhang10a}. We define:
\begin{equation}
\begin{aligned}
T_1 &:= g^{-1}_1(Y) \textnormal{ and } h_1 := f_2 \circ g_1 \\ 
T_2 &:= g^{-1}_2(X) \textnormal{ and } h_2 := f_1 \circ g_2
\end{aligned}
\end{equation}
Since $g_2$ is invertible, the independence between $X$ and $E_1$ is equivalent to the independence between $T_2$ and $E_1$. Similarly, the independence between $Y$ and $E_2$ is equivalent to the independence between $T_1$ and $E_2$. Consider the transformation $F:(E_2,T_1)\mapsto (E_1,T_2)$:
\begin{equation}
\begin{aligned}
E_1 &= T_1 - f_1(X) = T_1 - f_1(g_2(T_2)) \\
T_2 &= f_2(Y) + E_2 = f_2(g_1(T_1)) + E_2 
\end{aligned}
\end{equation}
The Jacobian matrix of this transformation is given by:
\begin{equation}
\begin{aligned}
\textnormal{J} &:= \textnormal{J}(F(e_2,t_1)) \\
&=
\left[
\begin{array}{c|c}
-\textnormal{J}(h_2(t_2))  & I - \textnormal{J}(h_2(t_2)) \cdot \textnormal{J}(h_1(t_1)) \\
\hline
I & \textnormal{J}(h_1(t_1))
\end{array}
\right]
\end{aligned}
\end{equation}
Since $I$ commutes with any matrix, by Thm.~3 in~\citep{Silvester_determinantsof}, we have:
\begin{equation}
\begin{aligned}
&\Big\vert \det(\textnormal{J}(F(E_2,T_1))) \Big\vert \\
=& \Bigg\vert \det\Big(-\textnormal{J}(h_2(T_2)) \cdot \textnormal{J}(h_1(T_1)) \\
&\;\;\;\; - I \cdot ( I - \textnormal{J}(h_2(T_2)) \cdot \textnormal{J}(h_1(T_1)) ) \Big) \Bigg\vert = 1
\end{aligned}
\end{equation}
Therefore, we have: $p_{T_2}(t_2) \cdot p_{E_1}(e_1) = p_{T_1,E_2}(t_1,e_2)/\vert \det\textbf{J} \vert = p_{T_1,E_2}(t_1,e_2)$. Hence, $\log (p_{T_1,E_2}(t_1,e_2)) = \eta_1(t_2) + \eta_2(e_1)$ and we have:
\begin{equation}
\begin{aligned}
\frac{\partial \log (p_{T_1,E_2}(t_1,e_2))}{ \partial e_2} &= \nabla \eta_1(t_2)  - \nabla \eta_2(e_1) \cdot \textnormal{J}(h_2(t_2))
\end{aligned}
\end{equation}
Therefore, 
\begin{equation}
\begin{aligned}
&\frac{\partial^2 \log (p_{T_1,E_2}(t_1,e_2))}{\partial e_2 \partial t_1} \\
=& \textnormal{H}(\eta_1(t_2))\cdot \textnormal{J}(h_1(t_1)) \\
&- \textnormal{H}(\eta_2(e_1)) \cdot \left( I - \textnormal{J}(h_2(t_2)) \cdot \textnormal{J}(h_1(t_1)) \right) \cdot \textnormal{J}(h_2(t_2)) \\
&-\nabla(\eta_2(e_1)) \cdot \textnormal{H}(h_2(t_2)) \cdot \textnormal{J}(h_1(t_1))\\
=&\textnormal{H}(\eta_1(t_2))\cdot \textnormal{J}(h_1(t_1)) - \textnormal{H}(\eta_2(e_1)) \cdot \textnormal{J}(h_2(t_2)) \\
&+ \textnormal{H}(\eta_2(e_1))\cdot \textnormal{J}(h_2(t_2)) \cdot \textnormal{J}(h_1(t_1)) \cdot \textnormal{J}(h_2(t_2))\\ 
&-\nabla(\eta_2(e_1)) \cdot \textnormal{H}(h_2(t_2)) \cdot \textnormal{J}(h_1(t_1))\\
\end{aligned}
\end{equation}
The independence between $T_1$ and $E_2$ implies that for every possible $(t_1,e_2)$, we have: $\frac{\partial^2 \log p_{T_1,E_2}(t_1,e_2)}{\partial e_2 \partial t_1}=0$. 
\end{proof}

\begin{lemma}[Reduction to post-linear models]\label{lem:reduce}
Let $f(x) = \sigma_1(W_d \dots \sigma_1(W_1x))$ and $g(u,v) = \sigma_2(U_k \dots \sigma_2(U_1(u,v) ))$ be two neural networks. Then, if $Y = g(f(X),E)$ for some $E \indep X$, we can represent $Y = \hat{g}(\hat{f}(X)+N)$ for some $N \indep X$.
\end{lemma}

\begin{proof}
Let $f(x) = \sigma_1(W_d \dots \sigma_1(W_1x))$ and $g(u,v) = \sigma_2(U_k \dots \sigma_2(U_1(u,v) ))$ be two neural networks. Here, $(u,v)$ is the concatenation of the vectors $u$ and $v$. We consider that $U_1(f(X),E) = U^1_1 f(X) + U^2_1 E$. We define a noise variable $N := U^2_1 E$ and have: $X \indep N$. In addition, let $\hat{f}(x) := U^1_1 f(x)$ and $\hat{g}(z) := \sigma_2(U_k \dots \sigma_2(U_2\sigma_2(z)))$. We consider that: $Y=\hat{g}(\hat{f}(X) + N)$ as desired.
\end{proof}

\identif*


\begin{proof}
Let $f_i(z) = \sigma_1(W_{i,d} \dots \sigma_1(W_{i,1}z))$ and $g_i(u,v) = \sigma_2(U_{i,k} \dots \sigma_2(U_{i,1}(u,v) ))$ (where $i=1,2$) be pairs of neural networks, such that, $\sigma_1$ and $\sigma_2$ are three-times differentiable. 
Assume that:
\begin{equation}
Y = g(f(X),E_1) \textnormal{ and } X = g(f(Y),E_2)
\end{equation}
for some $E_1 \indep X$ and $E_2 \indep Y$. By Lem.~\ref{lem:reduce}, we can represent 
\begin{equation}
\begin{aligned}
&Y=\hat{g}_1(\hat{f}_1(X) + N_1), \\
\textnormal{where }& N_1 = U^2_{1,1} E_1, \; \hat{f}_1 = U^1_{1,1} f_1(X) \\
\textnormal{and }& \hat{g}_1(z) = \sigma_2(U_{1,k} \dots \sigma_2 (U_{1,2} \sigma_2(z)))
\end{aligned}
\end{equation}
and also,
\begin{equation}
\begin{aligned}
&X=\hat{g}_2(\hat{f}_2(Y) + N_2), \\
\textnormal{where }& N_2 = U^2_{2,1} E_2, \; \hat{f}_2 = U^1_{2,1} f_2(X) \\
\textnormal{and }& \hat{g}_2(z) = \sigma_2(U_{2,k} \dots \sigma_2 (U_{2,2} \sigma_2(z)))
\end{aligned}
\end{equation}
Here, for each $i=1,2$ and $j=1,2$, $U^j_{i,1}$ are the sub-matrices of $U_{i,1}$ that satisfy:
\begin{equation}
U_{i,1}(f_i(X),E_i) = U^1_{i,1} f_i(X) + U^2_{i,1} E_i
\end{equation}
From the proof of Lem.~\ref{lem:reduce}, it is evident that the constructed $\hat{g}_1,\hat{f}_1$ and $\hat{g}_1,\hat{f}_2$ are three-times differentiable whenever $\sigma_1$ and $\sigma_2$ are. Therefore, by Lem.~\ref{lem:identif}, the following differential equation holds:
\begin{equation}\label{eq:diffeq}
\begin{aligned}
&\textnormal{H}(\eta_1(t_2))\cdot \textnormal{J}(h_1(t_1)) - \textnormal{H}(\eta_2(n_1)) \cdot \textnormal{J}(h_2(t_2)) \\
&+ \textnormal{H}(\eta_2(n_1))\cdot \textnormal{J}(h_2(t_2)) \cdot \textnormal{J}(h_1(t_1)) \cdot \textnormal{J}(h_2(t_2))\\ 
&-\nabla(\eta_2(n_1)) \cdot \textnormal{H}(h_2(t_2)) \cdot \textnormal{J}(h_1(t_1)) = 0
\end{aligned}
\end{equation}
where 
\begin{equation}
\begin{aligned}
&T_1 := \hat{g}^{-1}_1(Y) \textnormal{ and } h_1 := \hat{f}_2 \circ \hat{g}_1 \\
&T_2 := \hat{g}^{-1}_2(X) \textnormal{ and } h_2 := \hat{f}_1 \circ \hat{g}_2
\end{aligned}
\end{equation}
and $\eta_1(t_2) := \log p_{T_2}(t_2)$ and $\eta_2(n_1) := \log p_{N_1}(n_1)$.
\end{proof}

\sufficient*


\begin{proof}
Since $F$ and $f$ are invertible, one can represent: $F(X) = F(f^{-1}(f(X)))$ and $f(X) = f(F^{-1}(F(X)))$. Similarly, since $G$ and $g$ are invertible, we also have: $(F(X),R(Y)) \propto (f(X),E)$. Since $(F(X),R(Y)) \propto (f(X),E)$ and $F(X) \propto f(X)$, we have: $R(Y) = Q(F(X),E)$. However, $R(Y) \indep F(X)$ and therefore, we can represent $R(Y) = P(E)$ and vice versa.  
\end{proof}

\subsection{An Extension of Thm.~\ref{thm:sufficient}}

In this section we extend Thm.~\ref{thm:sufficient}. As a reminder, in our method, we employ two losses: a mapping loss $\mathcal{L}_{\textnormal{err}}(G,F,R)$ and a GAN-like independence loss $\mathcal{L}_{\textnormal{indep}}(R;D)$.  

Informally, in similar fashion to Thm.~\ref{thm:sufficient}, we would like to claim that when the algorithm successfully minimizes the losses, the information present in $r(Y) := E$ can be recovered from $R(Y)$. In Thm.~\ref{thm:sufficient}, it is shown that whenever the losses are optimal, we have: $R(Y) \propto r(Y)$. In Thm.~\ref{thm:sufficientExt}, we relax the optimality assumption and we would like to express the recoverability of $r(Y)$ given $R(Y)$ in terms of the success of the algorithm in minimizing the losses. By similar arguments we can also show that $f(X)$ can be recovered from $F(X)$.

To define a measure of recoverability of one random variable given another random variable we consider a class $\mathcal{T}$ of transformations $T:\mathbb{R}^n \to \mathbb{R}^{n}$. The reconstruction of a given random variable $V$ out of $U$ is given by:
\begin{equation}
\textnormal{Rec}_{\mathcal{T}}(V;U) := \inf_{T \in \mathcal{T}}\mathbb{E}_{(u,v) \sim (U,V)}[\|T(u) - v\|^2_2]
\end{equation}
The class $\mathcal{T}$ of transformations serves as the set of possible candidate mappings from $U$ to $V$. 

In our case, we are interested in measuring the ability to recover the information present in $r(Y)$ given $R(Y)$. Therefore, we would like to show that our algorithm implicitly minimizes:
\begin{equation}
\textnormal{Rec}_{\mathcal{T}}(r(y);R(y)) = \inf_{T \in \mathcal{T}} \mathbb{E}_{y \sim Y} [\| T(R(y)) - r(y) \|^2_2] 
\end{equation}
To do so, we upper bound the recoverability using the mapping error and a discriminator based divergence. In our bound, instead of employing $\mathcal{L}_{\textnormal{indep}}(R;D)$ directly, we make use of a different discriminator based measure of independence. For simplicity, we will assume that $\mathcal{T}$ consists of functions $T:\cup_{n\in\mathbb{N}}\mathbb{R}^n \to \mathbb{R}^{d_e}$ and for every fixed $u\in \mathbb{R}^k$, we have: $T_u(x) := T(x,u) \in \mathcal{T}$. This is the case of $\mathcal{T} = \cup_{n\in \mathbb{N}} \mathcal{T}_n$, where $\mathcal{T}_n$ is the class of fully-connected neural networks (with biases) with input dimension $n$ and fixed hidden dimensions. 

The proposed measure of independence will be based on the discrepancy measure~\citep{10.5555/507108,conf/colt/MansourMR09}. For a given class $\mathcal{D}$ of discriminator functions $D:\mathcal{X} \to \mathbb{R}$, we define the $\mathcal{D}$-discrepancy, also known as Integral Probability Metric~\citep{muller}, between two random variables $X_1$ and $X_2$ over $\mathcal{X}$ by:
\begin{equation}
\textnormal{disc}_{\mathcal{D}}[X_1 \| X_2] := \sup_{D \in \mathcal{D}} \left\{ \mathbb{E}_{x_1 \sim X_1} [D(x_1)] - \mathbb{E}_{x_2 \sim X_2} [D(x_2)] \right\}
\end{equation}
A well known example of this measure is the WGAN divergence~\cite{DBLP:conf/icml/ArjovskyCB17} that is specified by a class $\mathcal{D}$ of neural networks of Lipschitzness $\leq 1$.

In our bound, to measure the independence between $F(X)$ and $R(Y)$, we make use of the term:
\begin{equation}
\textnormal{disc}_{\mathcal{D}}\big[(F(X),R(Y),Y) \| (F(X'),R(Y),Y)\big]
\end{equation}
for some class of discriminators $\mathcal{D}$. Even though we do not use the original measure of independence, the idea is very similar. Instead of using a GAN-like divergence between $(X,R(Y))$ and $(X',R(Y))$, we employ a WGAN-like divergence between $(F(X),R(Y))$ and $(F(X'),R(Y))$. From a theoretical standpoint, it is easier to work with the discrepancy measure since it resembles a distance measure. 

The selection of $\mathcal{D}$ is a technical by-product of the proof of the theorem and one can treat it as an ``expressive enough'' class of functions. Specifically, each discriminator $D \in \mathcal{D}$ takes the following form: 
\begin{equation}
D(u_1,u_2,u_3) = \|T(u_1,u_2) - Q(u_3)\|^2_2
\end{equation}
where $T \in \mathcal{T}$ and $Q \in \mathcal{Q}$. Here, $u_1 \in \mathbb{R}^{d_f}$, $u_2 \in \mathbb{R}^{d_e}$ and $u_3 \in \mathbb{R}^{d_y}$. In particular, the discrepancy measure is:
\begin{equation}
\begin{aligned}
&\textnormal{disc}_{\mathcal{D}}\big[(F(X),R(Y),Y) \| (F(X'),R(Y),Y)\big] \\
=& \sup_{T \in \mathcal{T} ,Q \in \mathcal{Q}} \Big\{ \mathbb{E}_{(x,y)} \left[\|T(F(x),R(y)) - Q(y)\|^2_2\right] \\
&\quad\quad - \mathbb{E}_{(x',x,y)} \left[\|T(F(x'),R(y)) - Q(y)\|^2_2\right] \Big\}
\end{aligned}
\end{equation}
where $(x,y) \sim (X,Y)$ and $x' \sim X$ is an independent copy of $x$. A small discrepancy indicates that there is no discriminator $D \in \mathcal{D}$ that is able to separate between $(F(X),R(Y),Y)$ and $(F(X'),R(Y),Y)$. In particular, if $F(X) \indep R(Y)$, then, $\textnormal{disc}_{\mathcal{D}}\big[(F(X),R(Y),Y) \| (F(X'),R(Y),Y)\big] = 0$.

\begin{theorem}\label{thm:sufficientExt} Let $\mathbb{P}_{X,Y}$ admits a nonlinear model from $X$ to $Y$, i.e., $Y = g(f(X),E)$ for some random variable $E \indep X$. We denote by $\mathcal{G}$, $\mathcal{F}$ and $\mathcal{R}$ the classes from which the algorithm selects the mappings $G,F,R$ (resp.). Let $\mathcal{Q}$ be a class of $L$-Lipschitz continuous functions $Q:\mathbb{R}^{d_y} \to \mathbb{R}^{d_e}$ . Let $\mathcal{T}$ be be a class of functions that satisfies $\mathcal{Q} \circ \mathcal{G} \subset \mathcal{T}$. Let $\mathcal{D} = \left\{ D(u_1,u_2,u_3) := \| T(u_1,u_2) - Q(u_3) \|^2_2 \right\}_{Q \in \mathcal{Q}, T \in \mathcal{T}}$ be the class of discriminators. Then, for any $G \in \mathcal{G},F \in \mathcal{F}$ and $R\in \mathcal{R}$, we have:
\begin{equation}
\begin{aligned}
&\textnormal{Rec}_{\mathcal{T}}(r(Y);R(Y)) \\
\lesssim & \mathcal{L}_{\textnormal{err}}(G,F,R) + \lambda \\
&+ \textnormal{disc}_{\mathcal{D}}\big[(F(X),R(Y),Y) \| (F(X'),R(Y),Y)\big] \\
\end{aligned}
\end{equation}
where $\lambda := \inf_{Q \in \mathcal{Q}} \mathbb{E}_{y \sim Y}[\|Q(y) - r(y)\|^2_2]$.
\end{theorem}

As can be seen from Thm.~\ref{thm:sufficientExt}, when $\mathcal{Q}$ is expressive enough, such that, $\lambda$ is small and  $\mathcal{T}$ is expressive enough to satisfy $\mathcal{Q} \circ \mathcal{G} \subset \mathcal{T}$, for any functions $G,F,R$, the recoverability of $r(Y)$ given $R(Y)$ is upper bounded by the sum of the mapping error and the discriminator based independence measure. Hence, when selecting $G,F,R$ that minimize both losses, one implicitly learns a modeling $G(F(X),R(Y))$, such that, $r(Y)$ can be recovered from $R(Y)$. By a similar argument, the same relation holds for $f(X)$ and $F(X)$.

\begin{proof}
Let $Q^* \in \arg\min_{Q \in \mathcal{Q}} \mathbb{E}_{y \sim Y}[\|Q(y) - r(y)\|^2_2]$. We consider that:
\begin{equation}\label{eq:beforeDisc}
\begin{aligned}
&\inf_{T \in \mathcal{T}} \mathbb{E}_{(x,y) \sim (X,Y)}\| T(R(y)) - r(y) \|^2_2\\
\leq&  3 \inf_{T \in \mathcal{T}} \mathbb{E}_{(x,y) \sim (X,Y)}\| T(R(y)) - Q^*(y) \|^2_2\\
& + 3 \inf_{Q \in \mathcal{Q}}\| Q(y) - r(y) \|^2_2\\ 
=& 3 \inf_{T\in \mathcal{T}} \mathbb{E}_{(x,y) \sim (X,Y)}\| T(R(y)) - Q^*(y) \|^2_2 + 3 \lambda \\ 
=& 3 \inf_{T \in
\mathcal{T}} \mathbb{E}_{\substack{x'\sim X \\ (x,y) \sim (X,Y)}}\| T(F(x'),R(y)) - Q^*(y) \|^2_2 + 3\lambda \\
\end{aligned}
\end{equation}
where $x'$ and $x$ are two independent copies of $X$. The last equation follows from the fact that $x'$ and $y$ are independent and from the definition of $\mathcal{T}$, 
\begin{equation}
\begin{aligned}
&\inf_{T \in
\mathcal{T}} \mathbb{E}_{\substack{x'\sim X \\ (x,y) \sim (X,Y)}}\| T(F(x'),R(y)) - Q^*(y) \|^2_2 \\
\geq& \inf_{T\in \mathcal{T}} \mathbb{E}_{x'}\mathbb{E}_{(x,y) \sim (X,Y)}\| T_{F(x')}(R(y)) - Q^*(y) \|^2_2 \\
\geq& \mathbb{E}_{x'} \inf_{T\in \mathcal{T}} \mathbb{E}_{(x,y) \sim (X,Y)}\| T_{F(x')}(R(y)) - Q^*(y) \|^2_2 \\
\geq& \mathbb{E}_{x'} \inf_{T\in \mathcal{T}} \mathbb{E}_{(x,y) \sim (X,Y)}\| T(R(y)) - Q^*(y) \|^2_2 \\
=& \inf_{T\in \mathcal{T}} \mathbb{E}_{(x,y) \sim (X,Y)}\| T(R(y)) - Q^*(y) \|^2_2 \\
\end{aligned}
\end{equation}
Next we consider that for any $T \in \mathcal{T}$, we can rewrite:
\begin{equation}
\begin{aligned}
&\mathbb{E}_{\substack{x'\sim X \\ (x,y) \sim (X,Y)}}\| T(F(x'),R(y)) - Q^*(y) \|^2_2 \\
=& \mathbb{E}_{(x,y) \sim (X,Y)}\| T(F(x),R(y)) - Q^*(y) \|^2_2 \\
&+ \Big\{ \mathbb{E}_{\substack{x'\sim X \\ (x,y) \sim (X,Y)}}\| T(F(x'),R(y)) - Q^*(y) \|^2_2 \\ 
&\;\;\;\;\;\;- \mathbb{E}_{(x,y) \sim (X,Y)}\| T(F(x),R(y)) - Q^*(y) \|^2_2 \Big\} \\
\leq& \mathbb{E}_{(x,y) \sim (X,Y)}\| T(F(x),R(y)) - Q^*(y) \|^2_2 \\
&+ \textnormal{disc}_{\mathcal{D}}\big[(F(X),R(Y),Y) \| (F(X'),R(Y),Y)\big]
\end{aligned}
\end{equation}
Since the class $\mathcal{T}$ includes $Q^* \circ G$, we have:
\begin{equation}
\begin{aligned}
&\inf_{T} \mathbb{E}_{(x,y) \sim (X,Y)}\| T(R(y)) - r(y) \|^2_2 \\
\leq & 3 \mathbb{E}_{(x,y) \sim (X,Y)}\| Q^*(G(F(x),R(y))) - Q^*(y) \|^2_2 \\
&+ \textnormal{disc}_{\mathcal{D}}\big[(F(X),R(Y),Y) \| (F(X'),R(Y),Y)\big]
+ 3\lambda \\ 
\end{aligned}
\end{equation}
Since $Q^*$ is a $L$-Lipschitz function for some constant $L>0$, we have the desired inequality.
\end{proof}

\end{document}